\documentclass[12pt]{article}
\usepackage{fullpage}

\usepackage{amsmath,amsfonts,amssymb}
\usepackage{multirow}
\usepackage{cellspace}
\usepackage{setspace}
\usepackage{bm}
\usepackage{bbm}
\usepackage{subfigure}
\usepackage{graphicx}
\usepackage{tikz,pgfplots}
\usepackage{framed}
\usepackage{amsthm}
\usepackage{url}

\usepackage{algorithm}
\usepackage[noend]{algpseudocode}

\usepackage[round]{natbib}
\setcitestyle{authoryear,round,citesep={;},aysep={,},yysep={;}}

\usepackage{hyperref}
\hypersetup{colorlinks,
            linkcolor=blue,
            citecolor=blue,
            urlcolor=magenta,
           linktocpage,
           plainpages=false}

\usepackage[capitalise]{cleveref}

\usepackage{times}
\usepackage{graphicx} 
\usepackage{subfigure} 

\usepackage{natbib}

\usepackage{hyperref}

\usepackage{geometry}                
\usepackage[parfill]{parskip}    
\usepackage{subfigure}
\usepackage{graphicx}
\usepackage{amssymb}
\usepackage{epstopdf}
\newtheorem{theorem}{Theorem}[section]

\newtheorem{lemma}{Lemma}[section]

\newcommand{\tO}{\tilde{O}}

\newcommand{\A}{\mathcal{A}}

\newcommand{\E}{\mathbb{E}}
\newcommand{\M}{\mathcal{M}}
\newcommand{\Y}{\mathcal{Y}}
\newcommand{\est}[1]{\widetilde{#1}} 
\newcommand{\estbf}[1]{\mathbf{\widetilde{#1}}} 
\newcommand{\CnP}{\textsf{C\&P}}
\newcommand{\eS}{\mathcal{S}}

\newcommand{\floor}[1]{\lfloor #1 \rfloor}

\newcommand{\ignore}[1]{}

\def\bold0{\mathbf{0}}

\title{Multi-Player Bandits: The Adversarial Case }%

\author{%
Pragnya Alatur \footnote{Department of Computer Science, ETH Zurich. \newline \qquad \qquad
Emails: \texttt{pragnya.alatur@gmail.com}, \texttt{~yehuda.levy@inf.ethz.ch},\texttt{~krausea@ethz.ch}.}
\and 
Kfir Y. Levy\footnotemark[1] \and
Andreas Krause\footnotemark[1] 
}

\date{}

\begin{document}
\maketitle

\begin{abstract}
We consider a setting where multiple players sequentially choose among a common set of actions (arms). 
Motivated by a cognitive radio networks application, we  
assume that players incur a loss upon colliding, and that communication between players is not possible.
Existing approaches assume that the system is stationary.  Yet this assumption is often violated in practice, e.g., due to signal strength fluctuations. 
In this work, we design the first  Multi-player Bandit algorithm that provably works in arbitrarily changing environments, where the losses of the arms may even be chosen by an adversary.
This resolves an open problem posed by \citet*{MusicalChairs}.
\end{abstract} 
\newpage
\section{Introduction}
The Multi Armed Bandit (MAB) problem is a fundamental setting for capturing and analyzing sequential decision making. 
Since the seminal work of \citet{robbins1952some} there has been a plethora of research on this topic \citep{cesa2006prediction,bubeck2012regret,lattimore2018bandit},   addressing both the stochastic and adversarial MAB settings. In the \emph{stochastic} setting it is assumed that the environment is stationary, namely that except for noisy fluctuations, the environment does not change over time.
The \emph{adversarial} setting is more general, and enables to capture dynamical (arbitrarily changing) environments.

Most existing work on MABs considers a single player who sequentially interacts with the environment. Nevertheless, in many real world scenarios, the learner also \emph{interacts with other players}, either collaboratively or competitively. 
One such intriguing Multi-player setting arises in cognitive radio networks, where multiple broadcasters (players) share a common set of transmission channels (arms).  In this setting, players incur an extra loss upon colliding (transmitting on the same channel), and communication between players is generally not possible.
This challenging setting has recently received considerable attention,
\cite{avner2014concurrent,MusicalChairs,bistritz2018distributed}.

Despite  impressive progress on Multi-player Bandit problems, existing works only address the stochastic setting where the environment is stationary.  Yet, this may not capture common phenomena in cognitive radio networks, such  as channel  breakdowns or signal strength fluctuations due to changing environmental conditions.

In this work we address the adversarial Multi-player MAB setting, and provide the \emph{first efficient algorithm with provable guarantees.} This resolves an open problem posed by \citet*{MusicalChairs}. 
Concretely, assuming that $K$ players choose among a set of $N$ arms, our method ensures a total regret of $\tO(K^{4/3}N^{2/3}T^{2/3})$\footnote{Using $\tO(\cdot)$ we ignore logarithmic factors in $T,N$.}.

Our key algorithmic technique is to  imitate the idealized case where there is  full communication between the players.  Then, to address the no-communication constraint, we enforce the players to keep the same decisions (arms)  within long periods of time (blocks). This gives them the chance to coordinate between themselves via a simple protocol that uses  collisions as  a primitive, yet effective manner of communication.

\paragraph{Related Work}
The stochastic Multi-player MAB problem was extensively investigated in the past years.
The majority of work on this topic  assumes that  players may communicate with each other 
\citep{lai2008medium,liu2010distributed,vakili2013deterministic,liu2013learning,avner2016multi,avner2018multi}. 
The more realistic ``no-communication" setting was discussed by
\citet{anandkumar2011distributed,avner2014concurrent,MusicalChairs}, and \citet{bistritz2018distributed}.

\citet{avner2014concurrent} were the first to provide regret guarantees for the  ``no-communication"  stochastic setting, establishing a bound of $O(T^{2/3}$). This was later improved by \citet{MusicalChairs}, who established a constant regret  (\emph{independent} of $T$) for the case where there exists a  fixed gap between mean losses. Recently, \citet{bistritz2018distributed} have explored a more challenging setting, where each player has a different loss vector for the arms. They have provided an algorithm that ensures $O(\log^2 T)$ regret for this setting.

The case where the number of players may change throughout the game was  addressed by  
\citet{MusicalChairs}, where a regret bound of $O(\sqrt{T})$ is established. 
\citet{avner2014concurrent} also discuss this case and provide an algorithm that in some scenarios ensures an $O(T^{2/3})$ regret.

Different Multi-player adversarial MAB settings were explored by \citet{awerbuch2008competitive}, and \citet{cesa2016delay}. Nevertheless, these works allow players to communicate, and do not assume a ``collision loss".

Thus, in contrast to our ``no communication" adversarial setting, existing work either addresses the stochastic setting or allows communication.


\section{ Background and Setting}

\label{sec:Setting}

\subsection{Background}\label{sec:background}

The $N$-armed bandit setting can be described as a repeated game over $T$ rounds between a \emph{single} player and an adversary. At each round  $t \in[T]$ (we denote $[N]: = \{1,\ldots, N\}$, for any  $N\in \mathbb{Z}^{+}$),

\begin{enumerate}
	\item  the player chooses an arm $I^t \in [N]$
	\item  the adversary independently chooses a loss for each arm  $l_i^t \in [0,1],\; \forall i\in [N]$
	\item  the player incurs the loss of the chosen arm  $l^t_{I^t}$, and gets to  view the  loss of this arm only (bandit feedback)
\end{enumerate}
The goal of the player is to minimize the \emph{regret}, defined as,
\begin{align*}
R_T &:=  \sum_{t=1}^T l^t_{I^t} - \min_{i\in [N]} \sum_{t=1}^T l^t_{i}
\end{align*}
We are interested in learning algorithms that ensure an expected regret which is sublinear in $T$, here expectation is with respect to possible randomization in the player's strategy as well as in the choices of the adversary.

The seminal work of \citet{Exp3} presents an algorithm that achieves an optimal regret bound of  $O(\sqrt{T N \log N})$ for this setting. Their algorithm,  called EXP3, devises an unbiased estimate of the loss vector in each round, $\left\{\est{l}^t_{i}\right\}_{i\in[N] }$. These are then used to pick an arm  in each round by sampling, $I^t \sim \exp( -\eta \sum_{\tau=1}^{t-1}\est{l}^\tau_{i})$.

\subsection{K-Player MAB Setting}
We consider a repeated game of $T$ rounds between $K$ players and an adversary in the $N$-armed bandit setting. For now assume that each player has a unique \emph{rank} in $[K]$, and that each player knowns her own rank (but does necessarily know the rank of other players)\footnote{As we show in 
	Section~\ref{sec:kplayer}, such ranking can be achieved by running a simple procedure at the beginning of the game  (see  Algorithm~\ref{alg:ranking}).}. We also refer to the player with rank k as ``player k". 
Now at each round  $t \in[T]$,
\begin{enumerate}
	\item  each player $k\in[K]$ chooses an arm $I_k^t \in [N]$
	\item  the adversary independently chooses a loss for each arm  $l_i^t \in [0,1],\; \forall i\in [N]$
	\item for each player $k\in[K]$ one of two cases applies, 
	
	\textbf{Collision:} if another player chose the same arm, i.e., $\exists m\neq k$ such $I_k^t =I_m^t$, then player $k$ gets to know that a  collision occured, and incurs a loss of $1$. \\
	\textbf{No Collision:} if there was no collision, player $k$ incurs the loss of the chosen arm  $l^t_{I_k^t}$, and gets to view the loss of this arm only (bandit feedback).
\end{enumerate}
We emphasize that at each round all players play \emph{simultaneously}. We further  assume that communication between players is not possible. Finally, note that the ability to distinguish between collision and non-collision is a reasonable assumption when modelling cognitive radio networks and was also used in previous work, e.g. \citet{MusicalChairs}.

Our assumption is that the players are \emph{cooperative} and thus, their goal is to obtain low regret together with respect to the $K$ \emph{distinct} best arms in hindsight. Let $C_k^t\in \lbrace 0,1 \rbrace$ be an indicator for whether player $k$ collided at time $t$ ($C_k^t=1$) or not ($C_k^t=0$). With this, we define the regret $R_T$, after $T$ rounds as follows:
\begin{align*}
R_T &:= \underbrace{\sum_{t=1}^T{\sum_{\substack{k=1,\\ C_k^t=0}}^K{l_{I_k^t}^t}}}_{\text{no collisions}} + \underbrace{\sum_{t=1}^T{\sum_{k=1}^K{C_k^t}}}_{\text{collisions}} - \min_{\substack{i_1,...,i_K \in [N] \\ i_m \neq i_n,\forall m\neq n}}\sum_{t=1}^T{\sum_{k=1}^K{l_{i_k}^t}}
\end{align*}
We are interested in learning algorithms that ensure an expected regret which is sublinear in $T$.

\paragraph{Staying quiet} For simplicity, we assume that a player may choose to stay \emph{quiet}, i.e., not choose an arm, in any given round. By staying quiet she does not cause any collisions, but she will still suffer a loss of 1 for that round. This is a reasonable assumption when thinking about communication networks, as a user may choose to not transmit anything.
\paragraph{Adversary} 
For simplicity we will focus our analysis  on \emph{oblivious} adversaries, meaning that the adversary may know the strategy of the players, yet he is limited to choosing the loss sequence before the game starts. As we comment later on, using standard techniques we may extend our algorithm and analysis to address {non-oblivious} adversaries.
\paragraph{Further assumptions} We assume that every player knows $T$, the number of arms $N$, the number of players $K$ and that $K<N$ and $N<T$. Furthermore, we assume that the set of players is fixed and no player enters or exits during the game. Using standard techniques we may extend our method for the case where $T$ is unknown.

\newpage

\section{Idealized, Communication-Enabled Setting}
\label{sec:metaplayer}
Here we first discuss the  idealized setting, in which players are  able to fully communicate. Thus, they can coordinate their choices to avoid collisions, resulting in a collision loss of  $0$, i.e., $\sum_{t=1}^T{\sum_{k=1}^K{C_k^t}}=0$. In this case, the $K$ players would behave as a single player who chooses $K$ distinct arms in each round and aims to obtain low regret with respect to the $K$ best arms in hindsight.

Let us refer to such a hypothetical player  as  a \emph{K-Metaplayer}.
Similarly to the standard bandit setting, in each step $t$ this Metaplayer chooses $K$ distinct arms
$I^t := \lbrace I_1^t,...,I_K^t \rbrace$, and then gets to view the losses of these arms as  bandit feedback\footnote{
	Actually, as we will soon see, we analyze a slightly different setting where the Metaplayer gets to view only a  single arm chosen uniformly at random from $I^t$.}. 
Her regret $R_T^{meta}$ after $T$ rounds is defined with respect to the  best \emph{distinct}  $K$arms in hindsight  as follows:
\begin{align*}
R_T^{meta} &:= \sum_{t=1}^T{\sum_{k=1}^K{l_{I_k^t}^t}} -\min_{\substack{i_1,...,i_K \in [N] \\ i_m \neq i_n,\forall m\neq n}}\sum_{t=1}^T{\sum_{k=1}^K{l_{i_k}^t}}
\end{align*}
where for the best in hindsight we assume $i_m\neq i_n$ for any $m\neq n$.
What should the K-Metaplayer do to achieve sublinear regret?

For $K=1$, i.e., the traditional single-player N-armed bandit setting, we know that the EXP3 algorithm by \citet{Exp3} achieves an expected regret of $O(\sqrt{T N \log N})$. As we will see soon, we can adapt    EXP3 for the K-Metaplayer case: The idea is to view each subset of $K$ distinct arms $\lbrace I_1,...,I_K\rbrace$ as a \emph{single meta-arm} $I$. We define the set of meta-arms $\M$ as follows:
\begin{align*}
\M &:= \Big\lbrace \lbrace i_1,...,i_K \rbrace \subseteq [N] \Big| i_m \neq i_n \text{ for any } m\neq n \Big\rbrace
\end{align*}

We  further define the loss $\mathbf{l_I^t}$ of a meta-arm $I$ at time $t$ as:
\begin{align}\label{eq:MetaLoss}
\mathbf{l_I^t} &:= \sum_{k\in I}{l_k^t}
\end{align}
From this, it is immediate that the best meta-arm in hindsight w.r.t. losses $(\mathbf{l_I^t})_{I\in \M,t\in [T] }$ consists of the $K$ best arms in hindsight w.r.t. $(l_i^t)_{i\in [N], t\in [T]}$.

With these definitions, the K-Metaplayer essentially chooses one meta-arm from $\M$ in each step, receives that meta-arm's loss and aims to obtain low regret with respect to the best meta-arm in hindsight. This is very similar to the  traditional single-player multi-armed bandit setting, for which we know that EXP3 achieves low regret. Unfortunately, for the K-Metaplayer case, the number of arms  is $|\M|=\binom{N}{K}$, which is \emph{exponential} in $K$. Thus, directly applying  EXP3  in this setting would yield regret guarantees that also scale exponentially with $K$.

Fortunately, as seen in Eq.~\eqref{eq:MetaLoss}, the losses of different meta-arms are \emph{dependent} in each other. This means that viewing the loss of a single arm can be used to receive feedback for many meta-arms (concretely for exactly $|\M|\cdot K/N$ meta-arms).
In Alg.~\ref{alg:metaplayer}, we show an adaptation of  EXP3 that makes use of this structure to substantially reduce the regret compared to a straightforward  application of EXP3. \\
\textbf{Feedback model:} In Alg.~\ref{alg:metaplayer}, we assume  more restrictive bandit feedback, where the player gets to view only a  \emph{single arm chosen uniformly at random} (u.a.r.) among $I^t := \lbrace I_1^t,...,I_K^t \rbrace$. As we show later, this serves as a building block for our algorithm  in the more realistic \emph{no-communication} setting.

Let us shortly describe Alg.~\ref{alg:metaplayer}. For each arm $i\in[N]$, we hold an unbiased estimate of its cumulative loss $\sum_{\tau=1}^{t-1}\est{l_i^{\tau}}$. This is then directly translated to cumulative loss estimates for each meta-arm  $\{ \estbf{L_I^t}\}_{I\in\M}$. Then, similarly to EXP3, we sample each meta-arm proportionally  to 
$I^t\propto \exp(-\eta\estbf{L_I^t})$, and as  bandit feedback we  view one of the arms in $I^t$ chosen uniformly from this set. This feedback is then used to devise an unbiased estimate for the loss of all arms, $\{ \est{l_i^t} \}_{i\in[N]}$.

The following Lemma states the guarantees of Alg.~\ref{alg:metaplayer},

\begin{algorithm}[tb]
	\caption{K-Metaplayer algorithm (Input: $\eta$)} \label{alg:metaplayer}
	\begin{algorithmic}[1]
		\State {\bfseries Input:} $\eta$
		\For{$t=1$ {\bfseries to} $T$}
		\State Set cumulative loss estimate $\estbf{L_I^t} = \sum_{\tau=1}^{t-1}{\sum_{i\in I}{\est{l_i^\tau}}}$, for all meta-arms $I\in \M$
		\State Set probability $p^t(I) = \frac{e^{-\eta \estbf{L_I^t}}}{\sum_{J\in \M}{e^{-\eta \estbf{L_J^t}}}}$, for all meta-arms $I\in \M$ \label{eq:exp3_prob}
		\State Sample meta-arm $I^t=\lbrace I_1^t,...,I_K^t\rbrace$ at random according to $P^t=(p^t(I))_{I\in \M}$
		\State Pick one of the $K$ arms $J^t \in_{u.a.r.}\lbrace I_1^t,...,I_K^t\rbrace$
		\State Choose arms $I_1^t,...,I_K^t$ in the game, suffer losses $l_{I_1^t}^t,...,l_{I_K^t}^t$ and observe $l_{J^t}^t$
		\State Set loss estimate $\est{l_i^t}=K \cdot \frac{l_i^t}{\sum_{i\in I \in \M}{p^t(I)}} \cdot \mathbb{I}_{\lbrace J^t=i\rbrace}$, for all arms $i\in [N]$
		\EndFor
	\end{algorithmic}
\end{algorithm}

\begin{lemma}
	\label{lem:metaplayer_regret}
	Employing the  K-Metaplayer algorithm (Alg.~\ref{alg:metaplayer}) with $\eta=\sqrt{\frac{\log N}{T N}}$ guarantees a regret bound of $2K\sqrt{T N \log N}$.
\end{lemma}

\begin{proof}[Proof Sketch]
	Using the view of meta-arms, Alg. ~\ref{alg:metaplayer} is very similar to playing EXP3 on $|\M|=\binom{N}{K}$ meta-arms, as stated earlier. It can be shown that $\est{l_i^t}$ is an unbiased estimate of the true loss $l_i^t$ for every arm $i\in [N]$ and thus, by linearity of expectation, $\estbf{l_I^t}:=\sum_{i\in I}{\est{l_i^t}}$ is an unbiased estimate of $\mathbf{l_I^t}$ for all $I\in \M$. Given this and the observation that a meta-arm $I$ is chosen proportional to $\exp(-\eta \estbf{L_I^t})$ (see Alg. ~\ref{alg:metaplayer}), we can derive the following regret bound from standard EXP3 analysis:
	\begin{align*}
	\E[R_T^{meta}] &\leq \eta \sum_{t=1}^T{\sum_{I\in \M}{\E[p^t(I) \cdot \E[(\estbf{l_I^t})^2| p^t]}} + \frac{K \log N}{\eta}
	\end{align*}

The variance term $\E[(\estbf{l_I^t})^2|p^t]$ can be simplified by observing that $\est{l^t}=(\est{l_1^t},...,\est{l_N^t})$ has at most one non-zero entry, which implies that for any $j\neq k$, $\est{l_j^t}\cdot \est{l_k^t}=0$:
\begin{align*}
\E[(\estbf{l_I^t})^2 &|p^t] \\
&= \E[(\sum_{i\in I}{\est{l_i^t}})^2|p^t] \tag{by definition} \\
&= \sum_{j,k\in I}{\E[\est{l_j^t} \cdot \est{l_k^t} | p^t]} \tag{Linearity of expectation} \\
&= \sum_{i\in I}{\E[(\est{l_i^t})^2 | p^t]} \tag{all terms for $j\neq k$ cancel} \\
&= \sum_{i\in I}{\Bigg( \frac{K \cdot l_i^t}{\sum_{i\in Z\in \M}{p^t(Z)}}\Bigg)^2 \cdot Pr[J^t=i]} \\
&= \sum_{i\in I}{\Bigg( \frac{K \cdot l_i^t}{\sum_{i\in Z\in \M}{p^t(Z)}}\Bigg)^2 \cdot \underbrace{Pr[i\in I^t]}_{=\sum_{i\in Z\in \M}{p^t(Z)}} \cdot \frac{1}{K} } \\
&= K \sum_{i\in I}{\frac{(l_i^t)^2}{\sum_{i\in Z\in \M}{p^t(Z)}}}
\end{align*}

By plugging this result back into the regret expression and rearranging the summation terms, we conclude that $\E[R_T^{meta}]\leq 2 K \sqrt{T N \log N}$ for $\eta=\sqrt{\frac{\log N}{T N}}$. 
\end{proof}
For a detailed proof, we refer the reader to Section \ref{app:regret_metaplayer} in Appendix A.

\paragraph{Together as one K-Metaplayer}
Let us turn our attention back to the $K$ players in an idealized setting with \emph{full communication}. How do the players need to play in order to behave as the K-Metaplayer in Alg.~\ref{alg:metaplayer}?

We suggest to do so by assigning roles as follows:
Player 1 takes the role of a  global \emph{coordinator}, who decides which arm each of the $K$ players should pick. She samples $K$ arms in each step using the metaplayer algorithm, chooses one out of those $K$ u.a.r. for herself and assigns the rest to the other players. She then communicates to the other players what arms she has chosen for them. Players $2,...,K$ simply behave as \emph{followers} and accept whatever arm the coordinator chooses for them. With this, they are playing exactly as the metaplayer from Algorithm \ref{alg:metaplayer} and their regret would be bounded according to Lemma \ref{lem:metaplayer_regret}.

Note that the coordinator samples $K$ arms but receives  feedback only for   \emph{one}  of them.  This is the reason behind the feedback model  considered in Alg.~\ref{alg:metaplayer}.
Also, note that in this case, the coordinator is the only player that actually ``learns" from the feedback. All other players follow the coordinator  and ignore their loss feedbacks.


\section{Multi-Player MABs without Communication}
\label{sec:kplayer}
In the previous section, we described and analyzed an idealized setting where all players can fully communicate and can therefore act  as a single metaplayer.
Then we have shown  that by assigning Player 1  the role of a global \emph{coordinator}, and the rest of the players being \emph{followers}, we can exactly imitate the metaplayer algorithm. This strategy however,  requires full communication. Here, we show how to build on these  ideas  to devise an algorithm for the realistic ``no-communication"   setting.  Our $\CnP$ (Coordinate \& Play) algorithm is depicted in Figure~\ref{fig:kplayer_single_figure}, as well as in Alg.~\ref{alg:coordinator}, and~\ref{alg:follower}.
Its guarantees are stated in Theorem~\ref{thm:kplayer}.
And in Section~\ref{sec:Efficient} we  discuss an efficient implementation of our method.

Our method builds on top of the idealized scheme, with two additional ideas. 
\paragraph{Infrequent switches:} 
In order to give players the opportunity to coordinate, we prevent  them from frequently switching their decisions. 
Concretely, as is described in 
Fig~\ref{fig:kplayer_single_figure}, instead of sampling $K$ arms in each round, the coordinator (as well as the followers)  keeps the same $K$ arms for a block of $\tau$ consecutive rounds. The coordinator (Alg.~\ref{alg:coordinator}) runs a \emph{blocked version} of the K-metaplayer algorithm (Alg.~\ref{alg:metaplayer}):
In each block, the coordinator samples an arm according to Alg.~\ref{alg:metaplayer}, but stays on that arm for the entire block.  Then she feeds  the \emph{average} loss of that arm back into Alg.~\ref{alg:metaplayer} to update her loss estimates. 
While these blocks enable coordination,  they cause degradation to the regret guarantees \citep{DekelBlocking}. We elaborate on that in the analysis.

\begin{figure}[t]
	\centering
	\includegraphics[width=0.7\textwidth]{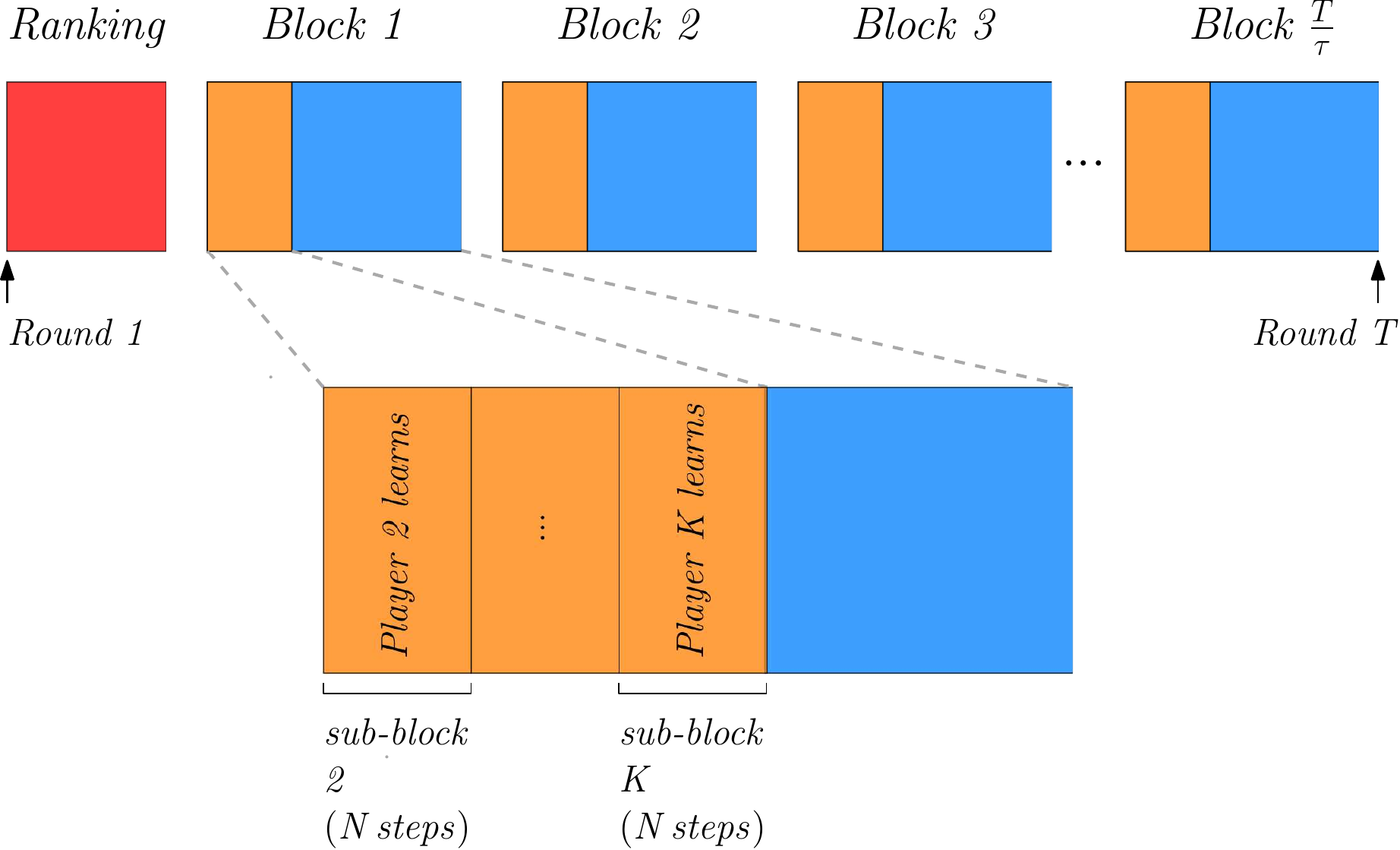}
	\caption{Illustration of the K-player algorithm. The upper part illustrates the timeline of the algorithm and the lower part shows the close-up view of a single block in the algorithm. \textit{Coordinate} phases are marked in orange and \textit{Play} phases are shown in blue. At the beginning of the algorithm, the players compute a ranking (red). This will be discussed further below.}
	\label{fig:kplayer_single_figure}
\end{figure}

\paragraph{Coordinate and Play}
We depict the timeline of our algorithm in Figure \ref{fig:kplayer_single_figure}.
As can be seen, we divide each block into two phases: \emph{Coordinate} phase (orange), and  \emph{Play} phase (blue).

At the beginning of each block, the coordinator picks $K$ arms according to the blocked version of the K-metaplayer algorithm. Then, during  \emph{Coordinate}, the coordinator assigns an arm to each of the $K-1$ followers.
Thus, the Coordinate phase is further divided into $K-1$ sub-blocks $2,...,K$ (Fig.~\ref{fig:kplayer_single_figure}, bottom part).
At sub-block $k$, the $k$'th follower gets assigned to an arm by a  protocol that uses collisions as a primitive, yet efficient, manner of communication.

This protocol  (see Alg.~\ref{alg:coordinator}, and Alg.~\ref{alg:follower}) is very simple: during sub-block $k$, the coordinator stays on the arm for player $k$ (a follower). Player $k$ tries out all arms in a round-robin fashion, until she collides with the coordinator. At this point, player $k$ learns her arm and the coordinator can repeat this procedure with the other players. While player $k$ is trying to find her arm, all other followers will stay quiet. Since each follower needs at most $N$ trials, all followers will have learnt their arms after $(K-1)\cdot N$ rounds.

After \emph{Coordinate}, each player has learnt her arm. During \emph{Play}, all players stay on their arms for the remaining steps of the block. At the end of the block, the coordinator uses the feedback she has collected in order to update her loss estimates.

If $T$ is not divisible by $\tau$, the players will play for blocks $1,...,\floor{\frac{T}{\tau}}$ and choose arms uniformly at random for the remaining steps. Since there will be less than $\tau$ steps left, this will increase the regret by at most $K \tau$.

\begin{algorithm}[tb]
	\caption{\CnP~Coordinator algorithm } \label{alg:coordinator}
	\begin{algorithmic}[1]
		\State {\bfseries Input:} $\eta$, block size $\tau$
		\For{block $b=1$ {\bfseries to} $\frac{T}{\tau}$}
		\Statex \textcolor{blue}{\text{Choose K arms according to the metaplayer}}
		\State Set cumulative loss estimate $\estbf{L_I^b} = \sum_{t=1}^{b-1}{\sum_{i\in I}}{\est{l_i^t}}$, for all meta-arms $I\in \M$
		\State Set probability $p^b(I) = \frac{e^{-\eta \estbf{L_I^b}}}{\sum_{J\in \M}{e^{-\eta \estbf{L_J^b}}}}$, for all meta-arms $I\in \M$
		\State Choose meta-arm $\widetilde{J^b}=\lbrace \widetilde{J_1^b},...,\widetilde{J_K^b}\rbrace$ at random according to $P^b = (p^b(I))_{I\in \M}$
		\State Let $\widetilde{I^b}=(\widetilde{I_1^b},...,\widetilde{I_K^b})$ be a uniform random permutation of $\widetilde{J^b}$
		\Statex \textcolor{blue}{\text{Coordinate}}
		\For{sub-block $r=2$ {\bfseries to} $K$} \Comment{Each sub-block has exactly $N$ steps}
		\State Choose $I_1^t=\widetilde{I_{r}^b}$  in steps $t$ until collision
		\State After collision, choose $I_1^t = \widetilde{I_1^b}$ for the remaining steps $t$ of sub-block $r$
		\EndFor
		\Statex \textcolor{blue}{\text{Play}}
		\State Choose arm $I_1^t=\widetilde{I_1^b}$ for remaining steps $t$ of block $b$
		\Statex \textcolor{blue}{\text{Feed average loss of arm $\widetilde{I_1^b}$ back to the metaplayer}}
		\State Set $\widehat{l_i^b} = \sum_{t=(b-1)\cdot \tau + 1}^{b\cdot \tau}{\mathbb{I}_{\lbrace I_1^t=i \rbrace} \cdot l_i^t}$, for all arms $i\in [N]$
		\State Set loss estimate $\est{l_i^b} = K \cdot \frac{\frac{1}{\tau}\widehat{l_i^b}}{\sum_{i\in I \in \M}{p^b(I)}} \cdot \mathbb{I}_{\lbrace \widetilde{I_1^b} = i\rbrace}$, for all arms $i\in [N]$
		\EndFor
	\end{algorithmic}
\end{algorithm}

\begin{algorithm}[tb]
	\caption{\CnP~Follower algorithm}
	\label{alg:follower}
	\begin{algorithmic}[1]
		\State {\bfseries Input:} block size $\tau$, rank $r$
		\For{block $b=1$ {\bfseries to} $\frac{T}{\tau}$}
		\Statex \textcolor{blue}{\text{Coordinate}}
		\State Stay quiet during sub-blocks $2,...,r-1$  \Comment{Each sub-block has exactly $N$ steps}
		\State During sub-block $r$, explore arms in a round-robin fashion until collision. $\widetilde{I_r^b}$ is the arm on which the collision occurred. Choose $\widetilde{I_r^b}$ for remaining steps of sub-block $r$.
		\State Stay quiet during remaining sub-blocks $r+1,...,K$
		\Statex \textcolor{blue}{\text{Play}}
		\State Choose $I_r^t=\widetilde{I_r^b}$ for remaining steps $t$ of block $b$
		\EndFor
	\end{algorithmic}
\end{algorithm}

\paragraph{Ranking}
So far we assumed that the players have unique ranks in $[K]$. They can compute the ranking by using a scheme that we adopt from \citet{MusicalChairs}.
The idea is playing a "Musical Chairs game"  on the \emph{first} $K$ arms $\{1,\ldots,K\}$ for $T_R$ rounds: A player chooses arms uniformly at random until she chooses an arm $i$ without colliding. At this point, that player becomes the owner of arm $i$ and will receive the rank $i$. This player $i$ then just stays on arm $i$ for the remaining of the $T_R$ rounds. We will set $T_R$ in a way that the ranking completes successfully with high probability.

\begin{algorithm}[tb]
	\caption{\CnP~Ranking}
	\label{alg:ranking}
	\begin{algorithmic}[1]
		\State {\bfseries Input:} $T_R$
		\For{$t=1$ {\bfseries to} $T_R$}
		\State Choose arm $r\in_{u.a.r.} [K]$
		\If{I did not collide} \Comment{My rank is $r$}
		\State Choose arm $r$ for the remaining of the $T_R$ rounds and return.
		\EndIf
		\EndFor
	\end{algorithmic}
\end{algorithm} $ $

The next theorem states the guarantees of our $\CnP$ Algorithm.
\begin{theorem}\label{thm:kplayer}
	Suppose that the $K$ players use our $\CnP$ Algorithm. Meaning,
	they first compute a ranking using Algorithm \ref{alg:ranking} with $T_R = K\cdot e\cdot\log T$. Afterwards, player 1 will act as coordinator and play according to Algorithm \ref{alg:coordinator}. The other players will behave as followers and run Algorithm \ref{alg:follower}. Then, the expected regret of the $K$ players is bounded as follows,
	$$
	\E [R_T] ~\leq~ 
	4 K^{4/3}N^{2/3}(\log N)^{1/3}T^{2/3} + 2K^2\cdot e\cdot\log T~,
	$$
	for block size $\tau=\Big(\frac{K^2 N T}{\log N}\Big)^{1/3}$ and $\eta=\sqrt{\frac{\log N}{\frac{T}{\tau} N}}$~.
\end{theorem}
As we mentioned earlier, our results apply to the oblivious case. Nevertheless, using a  standard mixing technique  \cite{Exp3} one can extend these results also for non-oblivious adversaries.

\begin{proof}[Proof of Theorem~\ref{thm:kplayer}]
By setting the length of the ranking phase $T_R = K\cdot e\cdot\log T$, the ranking completes after $T_R$ rounds with probability at least $1-\frac{K}{T}$ (see Section \ref{app:ranking_proof} in Appendix A for the derivation).
	
\paragraph{Case 1: Ranking unsuccessful}
With probability at most $\frac{K}{T}$, the players do not succeed in computing a ranking. The worst regret that they could obtain in the game is $K T$.
	
\paragraph{Case 2: Ranking successful}
In the idealized setting with communication from the previous section \ref{sec:metaplayer}, the \text{Coordinate} phase would not be necessary. In that case, Algorithms \ref{alg:coordinator} and \ref{alg:follower} together are just the result of applying the blocking technique to the K-Metaplayer algorithm \ref{alg:metaplayer}. This can be analyzed using the following Theorem from \citet{DekelBlocking},

\begin{theorem}\cite{DekelBlocking}\label{thm:dekel_blocking}
	Let $\A$ be a bandit algorithm with expected regret bound of $R(T)$.
	Then using the blocked version of $\A$ with a block of size $\tau$ gives a regret bound of 
	$\tau R(T/\tau) +\tau$.
\end{theorem}
The term $\tau$ above accounts for the additional regret in case $T$ is not divisible by $\tau$. Since we have $K$ players, we will replace that term by $K \tau$. Hence, by applying the above theorem to the regret bound from Lemma \ref{lem:metaplayer_regret}, we obtain that the regret of the $K$ players in a setting with communication would be $C\cdot T^{1/2}\tau^{1/2}+K \tau$ for $C=2 K \sqrt{N \log N}$.

In the real setting without communication, the \text{Coordinate} phase is needed and takes ${(K-1)  \cdot N}$ steps. During the \text{Coordinate} phase in one block, each player adds at most $(K-1) \cdot N$ to the total regret, either by staying quiet (loss 1) or by not choosing the optimal arm (round-robin exploration). Thus, the \text{Coordinate} phase increases the total regret by at most $\frac{T}{\tau} \cdot (K-1) \cdot N \cdot K$.

Finally, the ranking algorithm adds $K\cdot T_R=K^2\cdot e\cdot\log T$ to the regret. Put together, the expected regret of the $K$ players, assuming that ranking was successful (we denote this event  by $\eS$), is bounded as follows:
\begin{align*}
\E[&R_T | \eS]\\
& \leq
\underbrace{C \cdot T^{1/2}\tau^{1/2} + K\tau}_{\text{Thm. \ref{thm:dekel_blocking} + Lemma  \ref{lem:metaplayer_regret}}} + \underbrace{\frac{T}{\tau}\cdot K^2 N}_{\text{Coordinate}} 
+ \underbrace{K^2\cdot e\cdot\log T}_{\text{Ranking}} \\
&\leq 3 \cdot K^{4/3}N^{2/3}(\log N)^{1/3}T^{2/3} + \frac{K^{5/3}N^{1/3}}{(\log N)^{1/3}} T^{1/3} \\
&\quad
+K^2\cdot e\cdot\log T \\ 
&\leq 4 \cdot K^{4/3}N^{2/3}(\log N)^{1/3}T^{2/3} + K^2\cdot e\cdot\log T 
\end{align*}
where in the second line  we  use  $C=2 K \sqrt{N \log N} = (4K^2 N \log N)^{1/2}$ which holds  by Lemma \ref{lem:metaplayer_regret}; we also take  $\tau = \Big(\frac{K^2 \cdot N \cdot T}{\log N}\Big)^{1/3}$ and $\eta = \sqrt{\frac{\log N}{\frac{T}{\tau}N}}$. The last line  uses $K<T$.

Combining the results from cases 1 and 2 with $T_R=K\cdot e\cdot\log T$, gives the following bound:
\begin{align*}
\E[R_T] &= \underbrace{Pr[\eS]}_{\leq 1}\cdot \E[R_T | \eS] 
+ \underbrace{Pr[\eS^c]}_{\leq \frac{K}{T}} \cdot \underbrace{\E[R_T | \eS^c]}_{\leq K \cdot T} \\
&\leq 4 \cdot K^{4/3}N^{2/3}(\log N)^{1/3}T^{2/3} + K^2\cdot (e\cdot\log T +1) 
\end{align*}
where $\eS$ denotes the event where ranking is successful, and $\eS^c$ is its complement.
\end{proof}

\textbf{Remark:} So far we assumed that the players need to stay quiet during the Coordinate phase, but this assumption is actually not necessary. We will discuss in section \ref{app:quiet} in Appendix A how this assumption can be relaxed.

\subsection{Efficient Implementation}
\label{sec:Efficient}
In each block $b$ of algorithm \ref{alg:coordinator}, the coordinator samples a meta-arm $\widetilde{J^b} \in \M$ according to the probability distribution
\begin{align}
Pr[\widetilde{J^b} = I] &\propto \exp{(-\eta \estbf{L_I^b})} \nonumber\\
&= \exp{(-\eta \sum_{i\in I}{\est{L_i^b}})} \tag{$\est{L_i^b}=\sum_{t=1}^{b-1}{\est{l_i^t}}$} \nonumber \\
&= \prod_{i\in I}{\exp(-\eta \est{L_i^b})}
\label{eq:coord_prob}
\end{align}
As the number of possible outcomes is $|\M|=\binom{N}{K} = \Theta(N^K)$, computing the probability for each meta-arm naively would be expensive. Similarly, computing the marginal probability $\sum_{i\in I \in \M}{Pr[\widetilde{J^b}=I]}$ for an arm $i$ that is to be updated has cost $\Theta(N^{K-1})$ when done naively. By taking advantage of the structure in our probability distribution, we can show that sampling and marginalization can be made more efficient using a concept called \emph{K-DPPs}.

\emph{DPPs (Determinantal Point Processes)} \citep{Dpp} are probability distributions $\mathcal{P}: 2^\Y \rightarrow [0,1]$ (where $\Y=[N]$ is a fixed ground set and $2^\Y$ is the power set over $\Y$) that exhibit a particular structure: $\mathcal{P}$ can be specified in terms of the determinant of a so-called $N\times N$ kernel matrix. What makes DPPs appealing is that they allow us to sample from $\mathcal{P}$, i.e., subsets of $\Y$, in an efficient way, even if the outcome space is large.

While the output of a DPP can be \textit{any} subset of $\Y$, K-DPPs allow us to model particular distributions over the set of subsets of size \textit{exactly} $K$. For a K-DPP $\mathcal{P}$ with kernel matrix $L$, the probability of sampling a subset $Y$ of size $K$ is given by
\begin{align*}
\mathcal{P}(Y) &= \frac{\det (L_Y)}{\sum_{Y'\subseteq [N], |Y'|=K}{\det (L_{Y'})}}  \tag{Def. 5.1 of \citet{Dpp}}
\end{align*}
where $L_Y$ is the submatrix of $L$ indexed by the rows and columns in $Y$. A probability distribution that can be modeled as a K-DPP as specified above allows efficient sampling.

Going back to the coordinator, we observe that she samples a \textit{set} of $K$ arms from $[N]$. Furthermore, the probability $\mathcal{P}(I)$ of sampling a subset (meta-arm) $I\in \M$ is a function of the arms $i\in I$ as can be seen in Eq. ~\eqref{eq:coord_prob}. Since the probability for $I$ can be written as a product over (distinct) arms $i\in I$, this enables us to show that $\mathcal{P}(I)$ can be written as a determinant of a \emph{diagonal} matrix. Hence, with the following kernel matrix $L$, we obtain a K-DPP that models our coordinator's distribution as specified in Eq. ~\eqref{eq:coord_prob}:
\begin{align*}
L^b_{i,j} &= \begin{cases}
e^{-\eta \est{L_i^b}}, i = j \\
0, & i\neq j \text{ (diagonal matrix)}
\end{cases}
\end{align*}

By applying the guarantees provided by K-DPPs, the coordinator can implement the sampling efficiently as stated in Lemma \ref{lem:kdpps}.
$ $\newline

\begin{lemma}\label{lem:kdpps}
	Using K-DPPs, the cost of sampling a meta-arm in algorithm \ref{alg:coordinator} can be bounded by $O(N K)$ for any block. Similarly, the cost of computing the marginal probability $\sum_{i\in I\in \M}{p^b(I)}$ for a fixed arm $i$ is $O(N K)$. 
\end{lemma}
For the analysis, please refer to Section \ref{app:kdpps} in Appendix A.

The sampling cost that we state in Lemma~\ref{lem:kdpps} is strictly more efficient compared to the general case of sampling from K-DPPs. This is possible due to the special structure of the induced DPP in our case (where the kernel matrix $L$ is diagonal).


\section{Experiments}
\label{sec:experiments}
We run experiments with three different setups and compare the performance of our K-player algorithm to the Musical Chairs algorithm (MC) from \citet{MusicalChairs}. MC is designed for a reward-setting, while our algorithm uses losses. However, losses can easily be converted to corresponding rewards and vice versa by setting $r_i^t = 1-l_i^t$, where $r_i^t$ would be the reward of arm $i$ at time $t$.

MC achieves constant regret with high probability in a stochastic setting by assuming a fixed gap between the $K$-th and ($K+1$)-th arm. It starts with a \textit{learning} phase of $T_0\in O(1)$ rounds, during which players choose arms uniformly at random and observe rewards. At the end of the phase, players estimate the mean rewards of all arms based on the collected reward feedback. In the second phase, the players play a \textit{musical chairs} game, where each player chooses among the $K$ best arms according to her own estimates. As soon as a player chooses an arm without colliding for the first time, she becomes the owner of that arm and stays there for the rest of the game.

For all experiments, we set $N=8$, $K=4$, $T=240000$, $T_R=20$ and $T_0=3000$. This value for $T_0$ was also used for the experiments by \citet{MusicalChairs}. We repeat this for 10 runs for each setup and measure the online regret $R_t$, i.e., the difference between the cumulative player loss at time $t\in [T]$ and the cumulative loss of the $K$ arms that are the $K$ best in the time period $[t]$.

For each setup, we create a plot that shows the average regret and the standard deviation (as a colored region around the average). In the plots, the blue curves show the results of MC and the green curves show the results of our algorithm. The black dashed line indicates the end of MC's learning phase ($t=T_0$).

For all of the following three setups, we also run experiments to measure the accumulated regret $R_T$ after $T$ rounds. These  can be found in section \ref{app:experiments} of Appendix A.

\paragraph{Experiment 1}\label{par:experiment1}
We use a similar setup as in the experiments section of \citet{MusicalChairs}. First, we choose $N$ mean rewards in [0,1] u.a.r. with a gap of at least 0.05 between the $K$-th and $(K+1)$-th best arms. For each arm, the rewards are then  sampled i.i.d. from a Bernoulli distribution with the selected means. The results are shown in Figure \ref{fig:experiments_stochastic}.

\begin{figure}[t]
	\centering
	\includegraphics[width=0.7\textwidth]{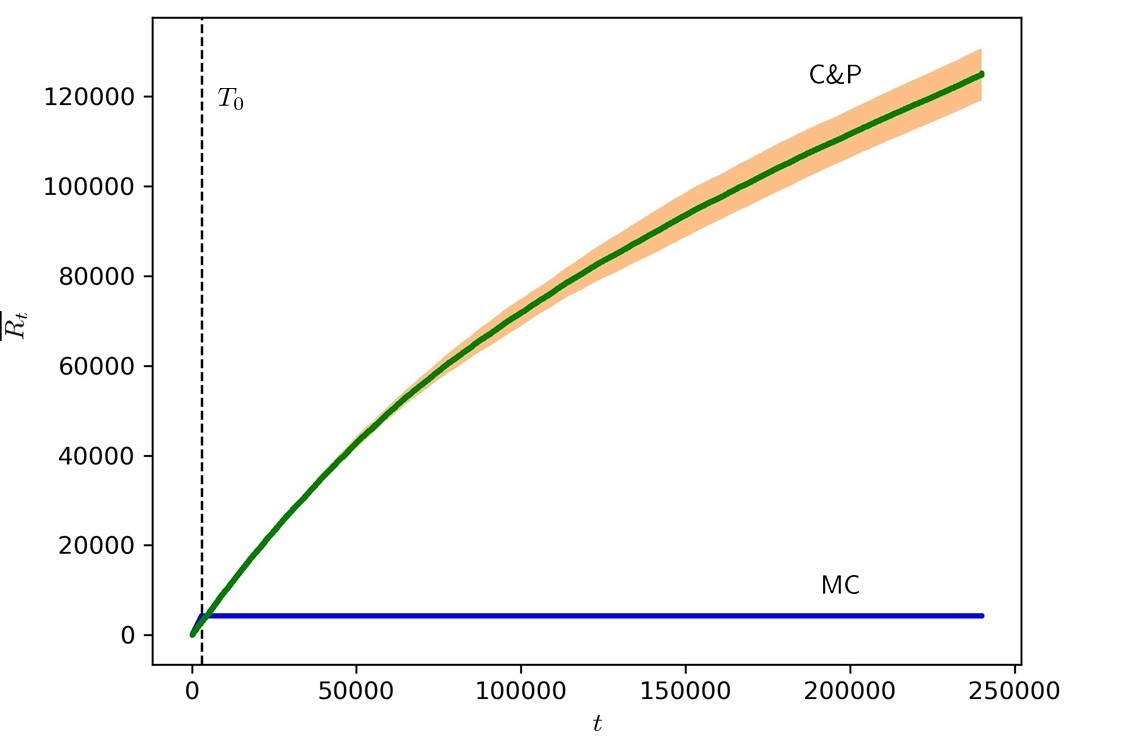}
	\caption{Results of experiment 1 (stochastic losses)}
	\label{fig:experiments_stochastic}
\end{figure}

As we can see, MC (blue curve) accumulates regret up to time $T_0$. But from that point onwards, after the musical chairs phase is done, the players are choosing optimally w.r.t. $K$ best arms in hindsight and thus their regret does not increase anymore. For our algorithm (green curve), we can see that it keeps accumulating regret until the end of the game.

\paragraph{Experiment 2}
In this experiment, we model a network scenario in which good links fail all of a sudden. Concretely, we initially set the mean loss $\mu_i$ for each arm $i$ as follows: $\mu_1=\mu_2=\mu_3=\mu_4=0.1$ and $\mu_5=\mu_6=\mu_7=\mu_8=0.3$. Each arm $i$'s losses are sampled i.i.d. from Bernoulli distribution $Ber(\mu_i)$.

At time $\frac{T}{4}$, ``link" (arm) 1 fails and its remaining losses are sampled i.i.d. from $Ber(0.9)$. After a while, at time $\frac{T}{3}$, link 3 also fails and from then on its losses are also chosen from $Ber(0.9)$. Figure \ref{fig:experiments_nonstochastic_fail} shows the result of this experiment. The red dashed lines represent the two link failures.

\begin{figure}[t]
	\centering
	\includegraphics[width=0.7\textwidth]{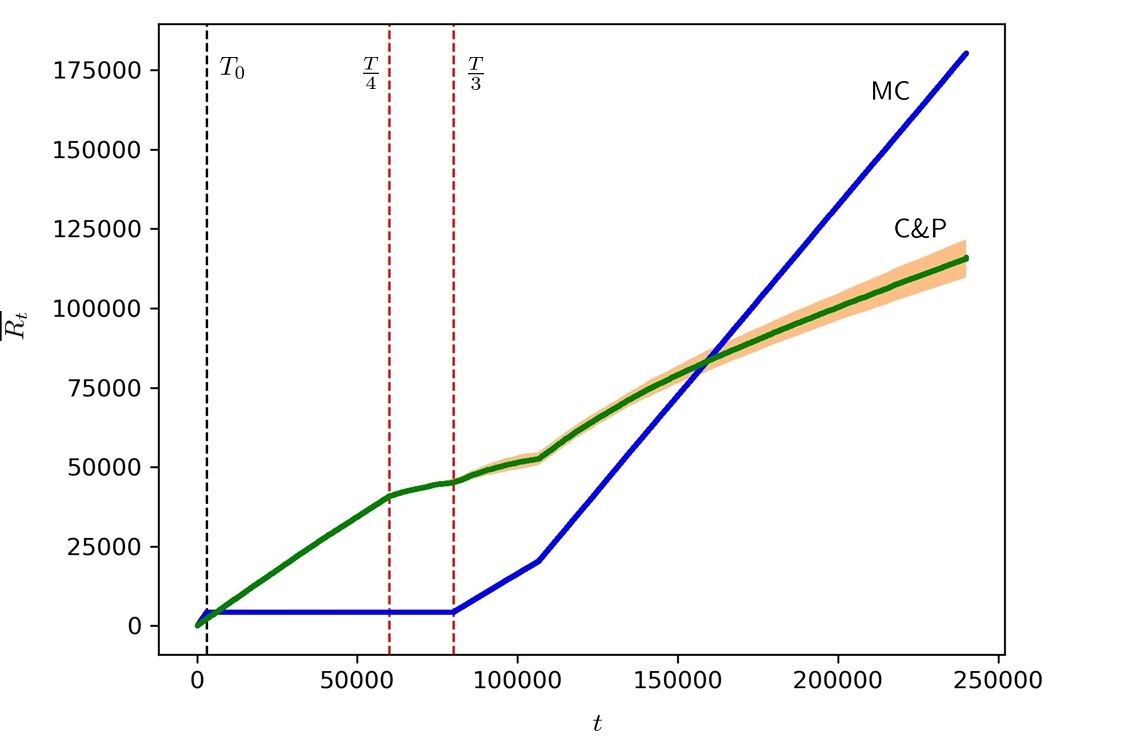}
	\caption{Results of experiment 2 (link failures). The red dashed lines indicate when the links failed.}
	\label{fig:experiments_nonstochastic_fail}
\end{figure}

As we can see, the link failures at times $\frac{T}{4}$ and $\frac{T}{3}$ happen after the learning phase in MC. Because of this, MC cannot react to them and its regret starts to increase. The green curve shows that while our algorithm initially has larger regret than MC, it is able to react to the link failures.
\paragraph{Experiment 3}
We model another network scenario, in which a bad link improves all of a sudden (or a link that was down comes up). We set the initial mean losses as follows: $\mu_1=0.9$ and $\mu_2=\mu_3=\mu_4=\mu_5=\mu_6=\mu_7=\mu_8=0.7$. As before, the losses are sampled i.i.d. from a Bernoulli distribution with the corresponding means. At time $\frac{T}{4}$, link 1 improves and its losses are from then on chosen from $Ber(0.1)$. Figure \ref{fig:experiments_nonstochastic_up} shows the results of this experiment.

\begin{figure}[t]
	\centering
	\includegraphics[width=0.7\textwidth]{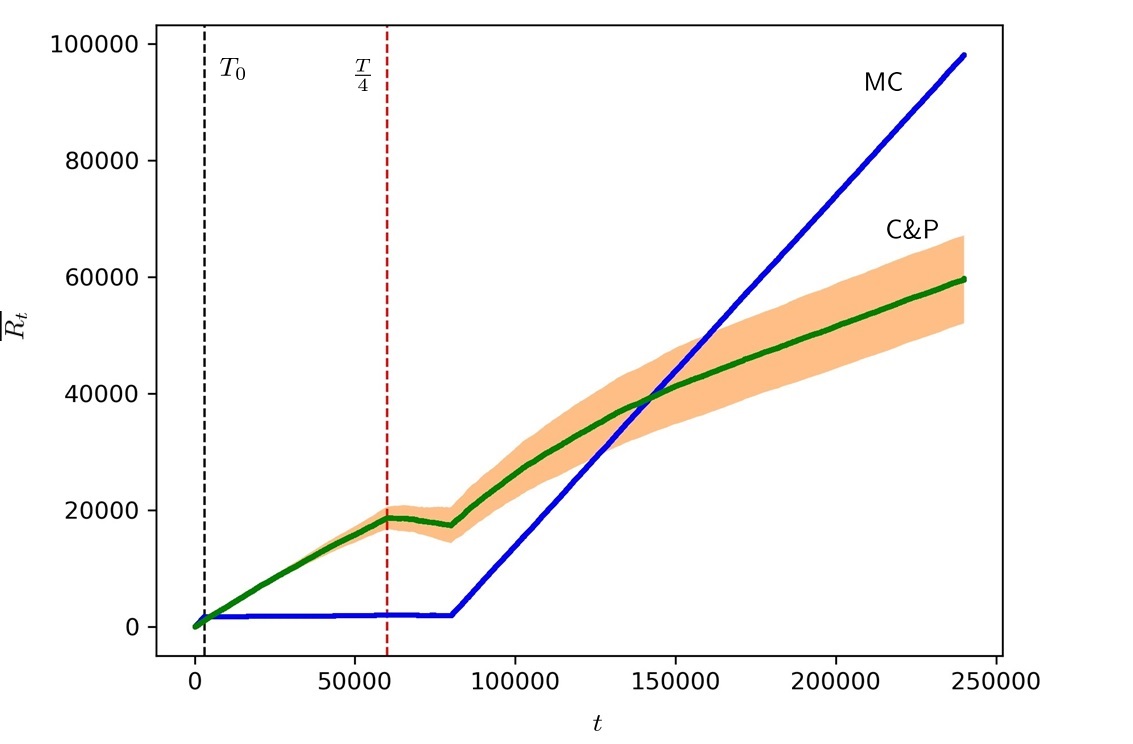}
	\caption{Results of experiment 3 (link improves). The red dashed line indicates when the link improved.}
	\label{fig:experiments_nonstochastic_up}
\end{figure}

Note that again the change in link quality happens after the learning phase in MC.


\section{Discussion and Conclusions}
\label{sec:Discussion}
We have presented an efficient algorithm for the multiplayer ``no communication" adversarial setting.
Our method obtains a regret bound of $\widetilde{O}(T^{2/3})$, and it is interesting to understand if this bound is tight or whether one can obtain a rate of ${O}(\sqrt{T})$ as in the single player setting.

In our algorithm, there is a single learner (coordinator) while all others just accept the coordinator's decisions and ignore the loss feedback that they receive. This poses a single point of failure. 	One possible way to remedy this might be to switch coordinators after each block in a round-robin fashion: Player 1 would be the coordinator in block 1, player 2 would be the coordinator in block 2 and so on.

\subsection*{Acknowledgement} 
We would like to thank Johannes Kirschner and  Mojm\'ir Mutn\'y for their valuable feedback on the manuscript.

This project has received funding from the European Research Council (ERC) under the European Union's Horizon 2020 research and innovation programme grant agreement No.~815943, as well as  from the ETH Zurich Postdoctoral Fellowship and Marie Curie Actions for People COFUND program.

\bibliography{bib}
\bibliographystyle{icml2017}

\appendix
\newpage
\section{Appendix}

\subsection{Regret analysis for Lemma \ref{lem:metaplayer_regret} (K-Metaplayer)} \label{app:regret_metaplayer}
In this section, we will prove that the metaplayer's regret with Alg. ~\ref{alg:metaplayer} is bounded by $\E[R_T^{meta}]\leq 2K\sqrt{T N \log N}$ for  $\eta=\sqrt{\frac{\log N}{T N}}$.

As stated in Lemma \ref{lem:metaplayer_regret}, by using the view of meta-arms, Alg. ~\ref{alg:metaplayer} is very similar to playing EXP3 on $|\M|=\binom{N}{K}$ meta-arms. In order to apply regret guarantees from the EXP3 analysis, we need to show that:
\begin{enumerate}
	\item A meta-arm $I\in \M$ is chosen proportional to $\exp(-\eta \estbf{L_I^t})$ at time $t$, where $\estbf{L_I^t}=\sum_{\tau=1}^{t-1}{\sum_{i\in I}{\est{l_i^\tau}}}$ is the cumulative loss estimate of $I$ at time $t$. This can be seen directly in Alg. ~\ref{alg:metaplayer}.
	\item $\estbf{l_I^t}=\sum_{i\in I}{\est{l_i^t}}$ is an unbiased estimate of the true meta-arm's loss $\mathbf{l_I^t}$ at time $t$, for any $I\in \M$ and any $t$. For this, we will first show that for any arm $i\in [N]$, $\est{l_i^t}$ is an unbiased estimate of $l_i^t$:
	\begin{align*}
	\E[\est{l_i^t}|p^t] &= K \cdot \frac{l_i^t}{\sum_{i\in Z\in \M}{p^t(Z)}} \cdot p^t(J^t=i) \\
	&= K \cdot \frac{l_i^t}{\sum_{i\in Z\in \M}{p^t(Z)}} \cdot \underbrace{Pr[i \in I^t]}_{=\sum_{i\in Z\in \M}{p^t(Z)}} \cdot \underbrace{Pr[J^t = i|i\in I^t]}_{=\frac{1}{K}} \\
	&= l_i^t
	\end{align*}
	From the law of total expectation, we can derive that $\E[\est{l_i^t}]=\E[\E[\est{l_i^t}|p^t]] = l_i^t$. Finally, by linearity of expectation (as $\estbf{l_I^t}=\sum_{i\in I}{\est{l_i^t}}$, we conclude that $\estbf{l_I^t}$ is an unbiased estimate of $\mathbf{l_I^t}$.
\end{enumerate}

Given 1. and 2., we can apply standard EXP3 regret guarantees to obtain the following bound on the metaplayer's regret:
\begin{align*}
\E[R_T^{meta}] &\leq \eta \sum_{t=1}^T{\sum_{I\in \M}{\E[p^t(I) \cdot \underbrace{\E[(\estbf{l_I^t})^2|p^t]}_{=:(*)}]}} + \frac{K \log N}{\eta} \tag{e.g. see Lecture 9, \citet{exp3_dekel}. Also, we used that $|\M|=\binom{N}{K}$.}
\end{align*}

The variance term $(*)$ can be simplified as follows:
\begin{align*}
\E[(\estbf{l_I^t})^2|p^t] &= \E[(\sum_{i\in I}{\est{l_i^t}})^2|p^t] \tag{by definition} \\
&= \sum_{j,k\in I}{\E[\est{l_j^t} \cdot \est{l_k^t} | p^t]} \tag{Linearity of expectation} \\
&= \sum_{i\in I}{\E[(\est{l_i^t})^2 | p^t]} \tag{The loss estimate at time $t$ is non-zero for at most one arm, thus all terms for $j\neq k$ cancel} \\
&= \sum_{i\in I}{\Bigg( \frac{K \cdot l_i^t}{\sum_{i\in Z\in \M}{p^t(Z)}}\Bigg)^2 \cdot Pr[J^t=i]} \\
&= \sum_{i\in I}{\Bigg( \frac{K \cdot l_i^t}{\sum_{i\in Z\in \M}{p^t(Z)}}\Bigg)^2 \cdot \underbrace{Pr[i\in I^t]}_{=\sum_{i\in Z\in \M}{p^t(Z)}} \cdot \underbrace{Pr[J^t=i|i\in I^t]}_{=\frac{1}{K}}} \\
&= K \sum_{i\in I}{\frac{(l_i^t)^2}{\sum_{i\in Z\in \M}{p^t(Z)}}}
\end{align*}

Plugging this back into our expression for the regret, we obtain:
\begin{align*}
\E[R_T^{meta}] &\leq \eta \sum_{t=1}^T{\sum_{I\in \M}{\E[p^t(I) \cdot \E[(\estbf{l_I^t})^2 | p^t]]}} + \frac{K \log N}{\eta} \\
&= \eta \sum_{t=1}^T{\sum_{I\in \M}{\E[p^t(I) \cdot K \sum_{i\in I}{\frac{(l_i^t)^2}{\sum_{i\in Z \in \M}{p^t(Z)}}}]}} + \frac{K \log N}{\eta} \\
&= K\eta \sum_{t=1}^T{\E[\underbrace{\sum_{I\in \M}{p^t(I) \sum_{i\in I}{\frac{(l_i^t)^2}{\sum_{i\in Z\in \M}{p^t(Z)}}}}}_{(\star)}]} \tag{Linearity of expectation} + \frac{K \log N}{\eta}
\end{align*}
In $(\star)$, we first sum over all meta-arms $I$ and then over all arms $i$ that are in $I$. We can instead sum over all arms $i$ first and then over all meta-arms $I$ that contain $i$. Hence, we can rewrite $(\star)$ as follows:
\begin{align*}
\sum_{I\in \M}{p^t(I)\sum_{i\in I}{\frac{(l_i^t)^2}{\sum_{i\in Z\in \M}{p^t(Z)}}}} &= \sum_{i=1}^N{\frac{(l_i^t)^2}{\sum_{i\in Z\in \M}{p^t(Z)}} \sum_{i\in I\in \M}{p^t(I)}} \\
&= \sum_{i=1}^N{(l_i^t)^2}
\end{align*}
By plugging this back into our regret expression, we get:
\begin{align*}
\E[R_T^{meta}] &\leq K\eta \sum_{t=1}^T{\E[\sum_{I\in \M}{p^t(I) \sum_{i\in I}{\frac{(l_i^t)^2}{\sum_{i\in Z\in \M}{p^t(Z)}}}}]} + \frac{K \log N}{\eta} \\
&= K\eta \sum_{t=1}^T{\E[\sum_{i=1}^N{(l_i^t)^2}]} + \frac{K \log N}{\eta} \\
&= K\eta \sum_{t=1}^T\sum_{i=1}^N{\E(\underbrace{ l_i^t}_{\in [0,1]})^2} + \frac{K \log N}{\eta}  \\
&\leq K T N \eta + \frac{K \log N}{\eta} \\
&= 2 K \sqrt{T N \log N} \tag{for $\eta = \sqrt{\frac{\log N}{T N}}$}
\end{align*}
This concludes the regret analysis for Lemma \ref{lem:metaplayer_regret}.

\subsection{Success analysis of the ranking algorithm \ref{alg:ranking}}\label{app:ranking_proof}
In this section, we will show that the players will successfully compute a ranking using algorithm \ref{alg:ranking} within $T_R=K\cdot e \cdot \log T$ rounds with probability at least $1-\frac{K}{T}$. The analysis uses ideas from the proof of Lemma 3 in \citet{MusicalChairs}.

For a fixed player, let $q^t$ be the probability that she gets a rank assigned in step $t$. $q^t$ can be bounded as:
\begin{align*}
q^t &= \sum_{i\in \text{Free}}{\frac{1}{K}\cdot (1-\frac{1}{K})^{\text{Unranked}-1}} \tag{\text{Free} = set of available arms at time $t$, \text{Unranked} = number of players who don't have a rank yet} \\
&\geq \sum_{i\in \text{Free}}{\frac{1}{K} \cdot (1-\frac{1}{K})^{K-1}} \tag{\text{Unranked} is at most $K$} \\
&\geq \frac{1}{K\cdot e} \tag{$|\text{Free}|\geq1$, $(1-\frac{1}{K})^{K-1}\geq e^{-1}$ for $K\geq 1$}
\end{align*}
The probability that she \text{doesn't} have a rank after step $t$ is thus at most:
\begin{align*}
& (1-\frac{1}{K\cdot e})^t \\
&\leq e^{- \frac{ t}{K\cdot e}} \tag{Using the inequality $1-x\leq e^{-x}$}
\end{align*}
By union bound, the probability that there's at least one player who is not fixed after $t=T_R$ rounds, is at most
\begin{align*}
K\cdot e^{- \frac{ T_R}{K\cdot e}}
\end{align*}
By setting $T_R = K\cdot e\cdot\log T$, we conclude that after $T_R$ rounds, the probability that all players have a rank, is at least
\begin{align*}
& 1-K\cdot e^{-K \cdot e \cdot \frac{\log T}{K\cdot e}} \\
&= 1 - \frac{K}{T}
\end{align*}

\subsection{Staying Quiet}
\label{app:quiet}
So far, we assumed that players need to stay quiet during the Coordinate phase of our $\CnP$ algorithm presented in section \ref{sec:kplayer}. I.e., during sub-block $k$, all players except the coordinator and player $k$, don't pick any arms. This assumption can however be relaxed using a simple modification to our protocol:

During sub-block $k\in \lbrace 2,...,K\rbrace$, all followers except player $k$ \textit{stay} on arm 1. Player $k$ explores all arms in a round-robin fashion for at most $N$ steps, until she collides on an arm $i \neq 1$. If she manages to do so, $i$ is the arm that the coordinator has chosen for her. If player $k$ doesn't collide on any other arm except on $1$, she can conclude that the coordinator has picked arm $1$ for her.

\subsection{Efficient sampling from the K-Metaplayer's distribution using K-DPPs (Lemma \ref{lem:kdpps})} \label{app:kdpps}
In this section, we will discuss how the coordinator can efficiently sample $K$ arms and compute marginal probabilities in Alg. ~\ref{alg:coordinator} using K-DPPs. We will first give some background on DPPs and K-DPPs before explaining how to use them for our case.

DPPs (Determinantal Point Processes) are probabilistic models that can model certain probability distributions of the type $\mathcal{P}: 2^{\Y} \rightarrow [0,1]$, where $\Y = [N]$ and $2^{\Y}$ is the power set of $\Y$.\footnote{In general, $\Y$ does not need to be discrete. For more information on the continous case, please refer to \citet{Dpp}.} Hence, a DPP samples subsets over a ground set $\Y$. In general, a DPP $\mathcal{P}$ is specified by a Kernel matrix (see definition 2.1 of \citet{Dpp}). L-Ensembles are a specific type of DPPs and we will focus only on those since this is what we will need for the coordinator algorithm. An L-Ensemble DPP $\mathcal{P}$ is defined by a $N\times N$-Kernel matrix $L$ as follows (see definition 2.2 of \citet{Dpp}):
\begin{align*}
\mathcal{P}(\mathbf{Y}=Y) \propto \det(L_Y) \tag{$Y\subseteq \Y$, $\mathbf{Y}$ is a random variable specifying the outcome of the DPP.}
\end{align*}
$L_Y$ is the submatrix of $L$ obtained by keeping only the rows and columns indexed by $Y$. The only restriction on $L$ is that it needs to be symmetric and positive semidefinite.

K-DPPs define probability distributions over subsets of size $K$, while the outcome set of a DPP can have any size. A K-DPP $\mathcal{P}^K$ is specified by a $N\times N$-Kernel matrix $L$ as follows (see definition 5.1 of \citet{Dpp}):
\begin{align*}
\mathcal{P}^K(\mathbf{Y}=Y) &= \frac{\det(L_Y)}{\sum_{Y' \subseteq \Y, |Y'|=K}{\det(L_{Y'})}}
\end{align*}

As before, $L$ needs to be positive and semidefinite. For DPPs and K-DPPs, sampling and marginalization can be done efficiently. Because of this, K-DPPs were appealing to us as they would allow us to efficiently sample a set of exactly $K$ \emph{distinct} arms, which is what we need for the coordinator. We will now see how we can model the coordinator's probability distribution over meta-arms as a K-DPP, i.e. we will determine how $L$ needs to be set.

Let us first recall the coordinator's probability for meta-arms. For this, let $\est{L_i^b}=\sum_{\tau=1}^{b-1}{\est{l_i^\tau}}$ denote the cumulative loss estimate for any arm $i\in \Y$ in block $b$. And let $\estbf{L_I^b}=\sum_{i\in I}{\est{L_i^b}}$ be the cumulative loss estimate for any meta-arm $I\in \M$ in block $b$. The probability that the coordinator chooses $I\in \M$ in block $b$, is:
\begin{align*}
p^b(I) &= \frac{e^{-\eta \estbf{L_I^b}}}{\sum_{J\in \M}{e^{-\eta \estbf{L_J^b}}}} \tag{see in Alg. ~\ref{alg:coordinator}}
\end{align*}
For our K-DPP, $\Y= [N]$ is the ground set and $\M$ the set of outcomes. For block $b$, let the $N\times N$-Kernel matrix $L^b$ be defined as follows:
\begin{align*}
L^b_{i,j} &= \begin{cases}
e^{-\eta \est{L_i^b}}, i = j \\
0, & i\neq j \text{ (diagonal matrix)}
\end{cases}
\end{align*}
Clearly, $L$ is symmetric and positive definite. Hence, it induces the following K-DPP $\mathcal{P}^K$:
\begin{align*}
\mathcal{P}^K(\mathbf{Y}=I) &\propto \det(L_I^b) \tag{$\mathbf{Y}$ is the random variable specifying the K-DPP's outcome, $I\in \M$} \\
&= \prod_{i\in I}{L_{i,i}^b} \tag{$L^b_I$ is a diagonal matrix} \\
&= e^{-\eta \sum_{i\in I}{\est{L_i^b}}} \\
&= e^{-\eta \estbf{L_I^b}} \tag{by definition of $\estbf{L_I^b}$}
\end{align*}
Note that a K-DPP samples \textit{subsets} of size $K$, i.e. $\mathbf{Y}$ does not contain any element twice and its size is $K$. 
Since the probabilities need to sum up to one, we conclude:
\begin{align*}
\mathcal{P}(\mathbf{Y} = I) &= \frac{e^{-\eta \estbf{L_I^b}}}{\sum_{J\in \M}{e^{-\eta \estbf{L_J^b}}}} \\
&= Pr[\widetilde{J^b}=I] \tag{Coordinator's probability of choosing meta-arm $I$}
\end{align*}

\paragraph{Cost for sampling a meta-arm}
Algorithm 1 in \citet{Dpp} describes how to sample from a general DPP. In a general DPP, the outcome can be any subset of $\Y$, its size is not necessarily equal to $K$. The algorithm consists of two phases:
\begin{enumerate}
	\item Sample eigenvectors of $L^b$. This determines the size of the DPP outcome.
	\item Use the sampled eigenvectors to actually choose a subset of $\Y$.
\end{enumerate}
For completeness, we have written down this algorithm here in Alg. ~\ref{alg:dpp}.

\begin{algorithm}[tb]
	\caption{Sampling from a DPP (Algorithm 1 in \citet{Dpp})}
	\label{alg:dpp}
	\begin{algorithmic}[1]
		\State $(v_n, \lambda_n)_{n=1}^N$ = Eigendecomposition of $L^b$
		\Statex \textcolor{blue}{Phase 1 begins}
		\State $J \gets \emptyset$
		\For{$n=1$ {\bfseries to} $N$}
		\State $J \gets J\cup \lbrace n \rbrace$ with probability $\frac{\lambda_n}{\lambda_n + 1}$
		\EndFor
		\Statex \textcolor{blue}{Phase 2 begins}
		\State $V \gets \lbrace v_n \rbrace_{n\in J}$
		\State $Y \gets \emptyset$
		\While{$|V|>0$}
		\State Select $i$ from $[N]$ with $Pr(i) = \frac{1}{|V|}\sum_{f\in V}{(v^T e_i)^2}$ \Comment{$e_i$ = i-th standard basis vector}
		\State $Y \gets Y \cup i$
		\State $V \gets V_{\bot}$, an orthonormal basis for the subspace of $V$ orthogonal to $e_i$
		\EndWhile
		Return $Y$
	\end{algorithmic}
\end{algorithm}

Since our matrix $L^b$ is diagonal, its eigendecomposition is very simple: The eigenvalues are simply the diagonal elements of $L^b$, the eigenvectors are the standard basis vectors. This means, that in phase 2, we would simply end up choosing only the elements in $J$, i.e. the returned set $Y$ is equal to $J$. Thus, we can actually finish after phase 1.

For a K-DPP, phase 1 of algorithm \ref{alg:dpp} is replaced with an algorithm that samples \text{exactly} $K$ eigenvectors. This then fixes the size of the outcome to $K$, which is what we want in a K-DPP. As we just saw, we actually only need phase 1 because our matrix $L^b$ is diagonal. Algorithm 8 in \citet{Dpp} describes how to sample exactly $K$ eigenvectors. Since this part requires $O(N K)$, we conclude that sampling a meta-arm in any block can be done in $O(N K)$.

\paragraph{Cost for computing the marginal probability of one arm}
If the K-Metaplayer decides to update arm $i$ at the end of block $b$, she needs to compute the marginal probability $\sum_{i\in I \in \M}{p^b(I)}$. We can rewrite this as follows:
\begin{align*}
\sum_{i\in I\in \M}{p^b(I)} &= \sum_{i\in I\in \M}{\frac{e^{-\eta \estbf{L_I^b}}}{Z_K^N}} \tag{Normalizer $Z_K^N := \sum_{J\in \M}{e^{-\eta \estbf{L_J^b}}}$} \\
&= \frac{1}{Z_K^N} \sum_{i\in I\in \M}{e^{-\eta \sum_{j\in I}{\est{L_j^b}}}} \tag{by definition of $\estbf{L_I^b}$} \\
&= \frac{e^{-\eta \est{L_i^b}}}{Z_K^N} \cdot \underbrace{\sum_{\substack{i\notin \lbrace i_1,...,i_{K-1} \rbrace \subseteq \Y}}{e^{-\eta \sum_{k=1}^{K-1}{\est{L_{i_k}^b}}}}}_{=:(*)}
\end{align*}

By inspecting the expression inside sum $(*)$ more closely, we observe that it looks very similar to the K-DPP that we defined before. In fact, that expression can be seen as a (K-1)-DPP over ground set $[N] \setminus \lbrace i \rbrace$ with Kernel matrix $L_{-i}^b$ consisting of $L^b$ without the i-th row and column. Therefore, the sum $(*)$ is actually just the normalization constant, let's call it $Z_{K-1}^{N-i}$, of that (K-1)-DPP. Hence, the marginal probability for arm $i$ can be written as:
\begin{align*}
\sum_{i\in I\in \M}{p^b(I)} &= \frac{e^{-\eta \est{L_i^b}}}{Z_K^N} \cdot Z_{K-1}^{N-i}
\end{align*}

From proposition 5.1 in \citet{Dpp}, we know that both $Z_K^N$ and $Z_{K-1}^{N-i}$ can be computed in $O(N K)$ each. We conclude that calculating the marginal probability for one arm in any block can be done in $O(N K)$.

\subsection{Experiments (Measuring the accumulated regret)}
\label{app:experiments}
For the three setups that we described in section \ref{sec:experiments}, we run experiments to measure the accumulated regret $R_T$ of both MC and our algorithm. We visualize the outcome in a loglog plot to compare the experimental results with our theoretical bound (Theorem \ref{thm:kplayer}).

In all three experiments, we set $N=8$, $K=4$, $T_R=25$ and $T_0=3000$ (length of MC's learning phase). For $T$, we choose $T=100000 + i\cdot 1000$, where $i\in \lbrace 0,...,1300\rbrace$. Per value of $T$, we do 10 runs and measure the regrets.

In the loglog plots, the blue dots show the average regrets of MC and the green dots the average regrets of our algorithm. The standard deviations are shown as coloured regions around the average regrets. Besides this, we fit a line on the log average regrets for each algorithm and plotted those as well. With these lines, we can compare whether the experimental results match what we expect from Theorem \ref{thm:kplayer}.

\paragraph{Experiment 1}
We use the same setup as in experiment 1 from \ref{sec:experiments}, i.e. arms with i.i.d. Bernoulli losses where the arms' means are sampled u.a.r. from [0,1] with a gap of at least 0.05 between the $K$-th and ($K+1$)-th best arm. The results are shown in Figure ~\ref{fig:experiments_acc_stochastic}.

\begin{figure}[t]
	\centering
	\includegraphics[width=0.7\textwidth]{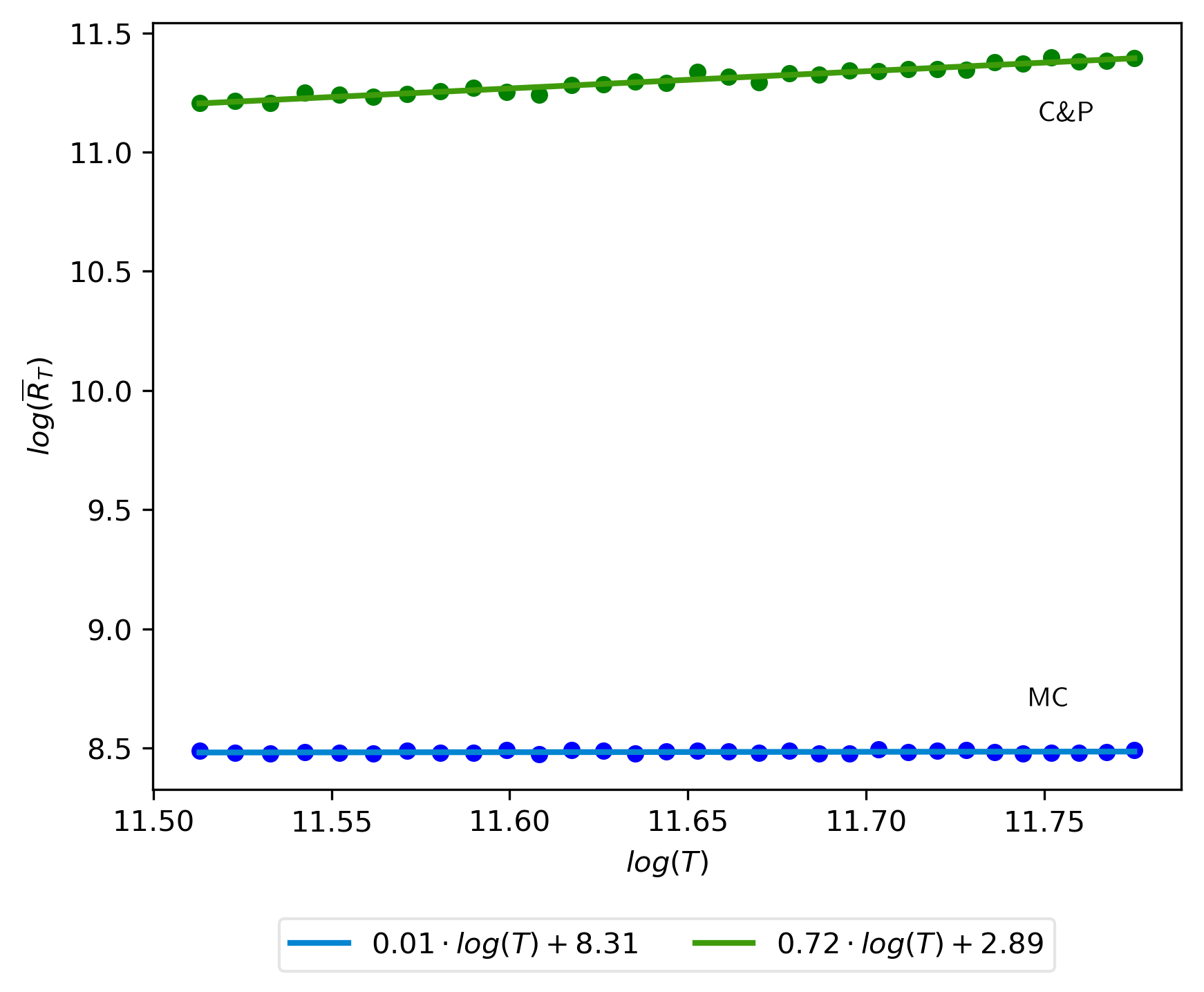}
	\caption{Loglog plot of experiment 1 (stochastic losses).}
	\label{fig:experiments_acc_stochastic}
\end{figure}


\paragraph{Experiment 2}
In this experiment, we use the setup from experiment 2 in section \ref{sec:experiments}, i.e. we model a network in which two links go down at time $\frac{T}{4}$ and $\frac{T}{3}$, respectively. Figure ~\ref{fig:experiments_acc_nonstochastic_down} shows the results of this experiment.

\begin{figure}[t]
	\centering
	\includegraphics[width=0.7\textwidth]{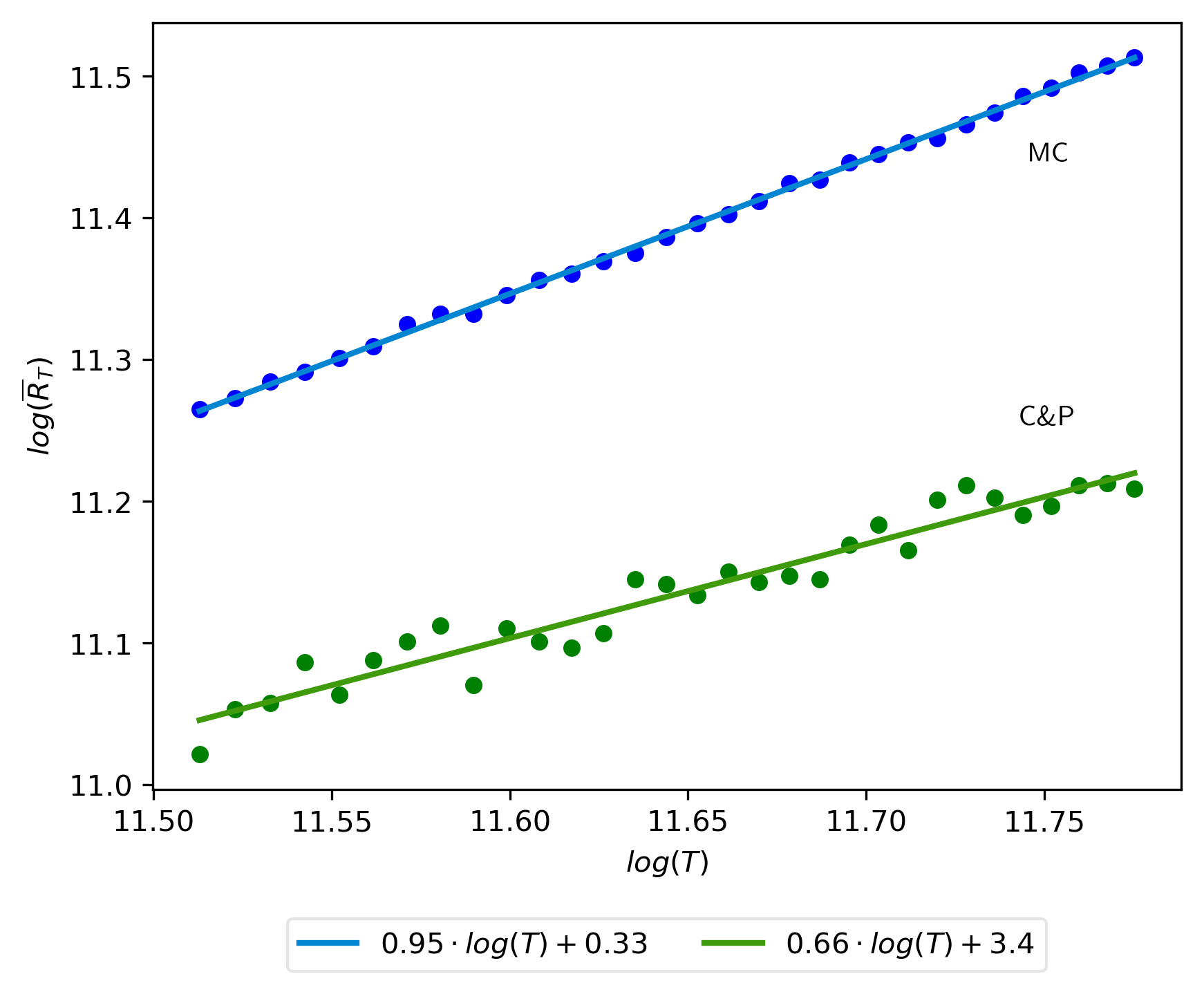}
	\caption{Loglog plot of experiment 2 (link failures).}
	\label{fig:experiments_acc_nonstochastic_down}
\end{figure}

\paragraph{Experiment 3}
For this, we use the setup from experiment 3 in section \ref{sec:experiments}, in which a bad link suddenly improves or comes up at time $\frac{T}{4}$. The outcome of this experiment shown shown in Figure ~\ref{fig:experiments_acc_nonstochastic_up}.

\begin{figure}[t]
	\centering
	\includegraphics[width=0.7\textwidth]{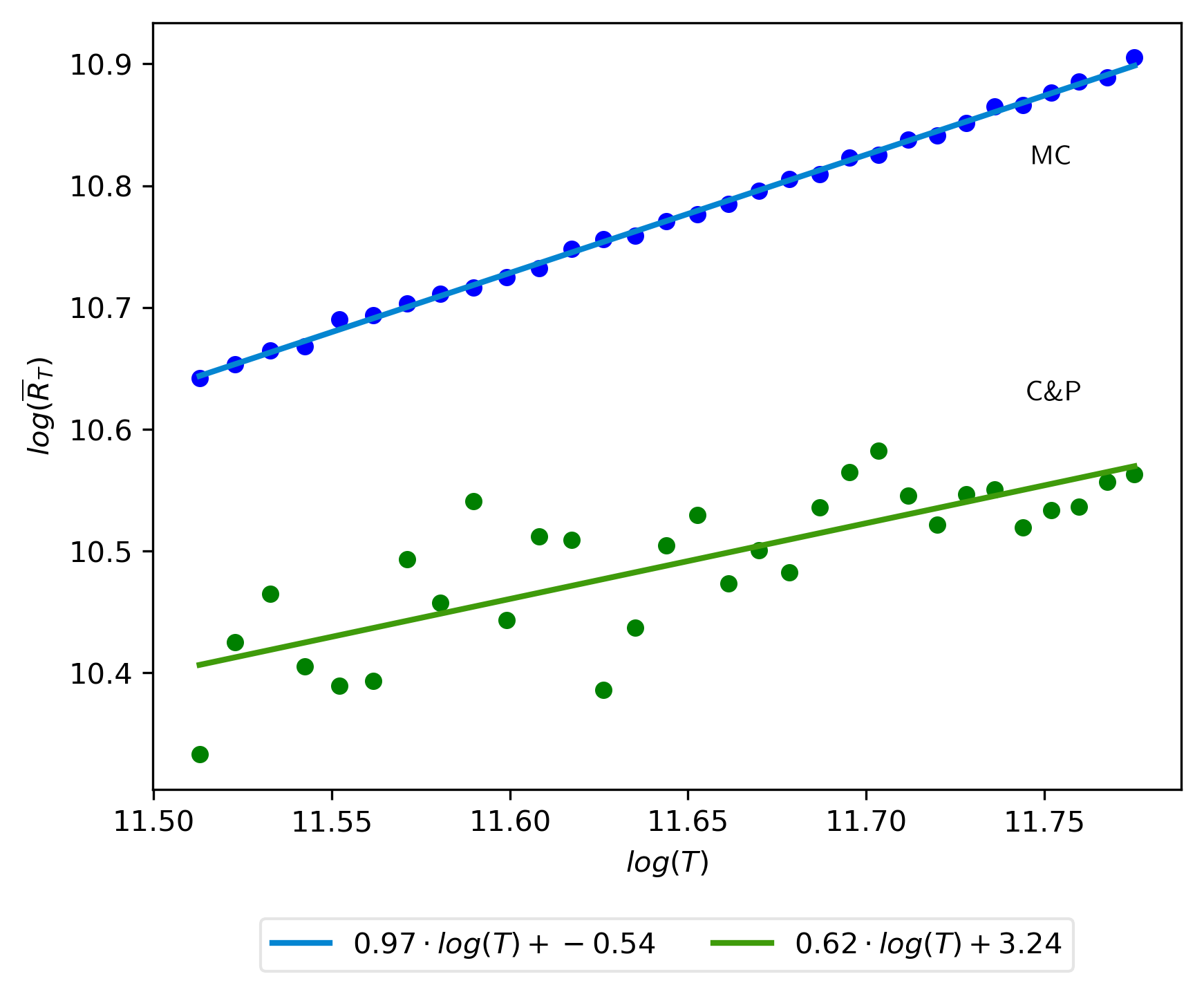}
	\caption{Loglog plot of experiment 3 (link improves).}
	\label{fig:experiments_acc_nonstochastic_up}
\end{figure}

\end{document}